\documentclass{article}
\usepackage{geometry}
\usepackage{amsmath, amsthm}
\usepackage{amsfonts}       
\usepackage{turnstile}
\usepackage{algorithm}      
\usepackage{algorithmic}    
\usepackage{authblk}
\usepackage{xcolor}
\usepackage{paralist}
\usepackage[round]{natbib}

\newtheorem{theorem}{Theorem}[section]
\newtheorem{lemma}[theorem]{Lemma}

\newtheorem{definition}[theorem]{Definition}

\def\X{\mathcal{X}}
\def\Xij{\mathcal{X}_{I_j}(\boldsymbol x^{(i)})}
\def\E{\mathbb{E}}
\def\yi{\boldsymbol y^{(i)}}
\def\yit{\boldsymbol y^{(i)\top}}
\def\bounda{\frac{w_j|I_j|(d+2)^4(\beta+\alpha + 7\alpha^4\sigma^2)}{\mu^2N^2\delta}}
\def\boundb{\frac{w_j(d+2)^4(\beta+\alpha + 7\alpha^4\sigma^2)}{\mu N\delta}}
\def\boundc{mC\sqrt{2^5\frac{w_j(d+2)^4(\beta+\alpha + 7\alpha^4\sigma^2)}{\mu N\delta}}}

\title{Conditional Linear Regression for Heterogeneous Covariances}
\author[]{Brendan Juba}
\author[]{Leda Liang}
\affil[]{Washington University in St. Louis \\ 
         \{bjuba, ledaliang\}@wustl.edu}
\date{}

\begin{document}

\maketitle

\begin{abstract}
 Often machine learning and statistical models will attempt to describe the majority of the data. However, there may be situations where only a fraction of the data can be fit well by a linear regression model. Here, we are interested in a case where such inliers can be identified by a Disjunctive Normal Form (DNF) formula. We give a polynomial time algorithm for the conditional linear regression task, which identifies a DNF condition together with the linear predictor on the corresponding portion of the data. In this work, we improve on previous algorithms by removing a requirement that the covariances of the data satisfying each of the terms of the condition have to all be very similar in spectral norm to the covariance of the overall condition.
\end{abstract}

\section{\uppercase{Introduction}}

Linear regression is a technique frequently used in statistical and data analysis. The task for standard linear regression is to fit a linear relationship among variables in a data set. Often, the goal is to find the most parsimonious model that can describe the majority of the data. In this work, we consider the situation where only a small portion of the data can be accurately modeled using linear regression. More generally, in many kinds of real-world data, portions of the data of significant size can be predicted significantly more accurately than by the best linear model for the overall data distribution: \cite{rgh+15} showed that there are attributes that are significant risk factors for gastrointestinal cancer in certain subpopulations, but not in the overall population. \cite{hjlw19} demonstrated that a variety of standard (real-world) regression benchmarks have portions that are fit significantly better by a different linear model than the best model for the overall data set; \cite{CJLLR20} presented further, similar findings. We will consider cases where linear regression fits well when the data set is conditioned on a simple condition, which is unknown to us. We study the task of finding such a linear model, together with a formula on the data attributes describing the condition, i.e., the portion of the data for which the linear model is accurate.

This problem was introduced by \cite{J17}, who gave an algorithm for conditional sparse linear regression, using the maximum residual as the objective. This was extended to the usual squared-error loss (as well as other $\ell_p$ losses) by \cite{hjlw19}. \cite{J17} also gave an algorithm for the general (non sparse) case that could only find a small fraction of the largest such condition. All of these algorithms find conditions describing subpopulations that are a union of some basic subsets of data, selected by ``terms.'' For example, simple families of terms may be obtained by considering the data for which a small set of categorical attributes take some specific values, or based on whether the value of some real attributes lie in specific quantiles of the distribution. \cite{CJLLR20} gave an algorithm for non sparse linear regression that matches the size of the largest condition, but only under a new assumption, that the covariances of the data satisfying each of the terms of the condition have to all be very similar in spectral norm to the covariance of the overall condition. 

Uniform covariances across terms is an extremely restrictive assumption: it means that essentially the only difference between populations selected by the various terms may be in their means. Note that the problem presupposes that there is significant heterogeneity in the conditional covariances across the distribution overall, or else the same linear model would be equally accurate across the various subsets; concretely, the risk factors found by \cite{rgh+15} are a kind of correlation between a factor and the target variable that exists in the identified subpopulation, but not in the larger population. In a real-world data set, there is no reason to expect that the only relationships that exist would involve the target attribute; relationships between other pairs of attributes may appear when we consider one term or another. For example, intuitively, if the data lies on a curved manifold, then conditioning on some attribute taking values that select one portion of the curve or another would alter such relationships since the tangent space changes, and the covariance matrix in any local region of the manifold only has eigenvectors lying in the tangent space. (Note that we only require a common component in the orthogonal subspace for a linear model to exist.)
So, Calderon et al's algorithm can only be guaranteed to find highly homogeneous subsets of a distribution that features significant heterogeneity overall. In their work, Calderon et al.\ concluded with the question of whether or not this new property was necessary to obtain a computationally efficient algorithm.

\subsection{Our Contribution}

In this work, we answer the questions posed by Calderon et al.\ and \cite{J17}, solving the form of the task originally sought: we give a polynomial-time algorithm that identifies a condition that covers as much of the distribution as the optimal condition and a linear model which provides a good fit when conditioned on said condition, even if the terms of the condition feature heterogeneous covariances. The only assumptions on the data we use are bounds on the moments of the data itself (including hypercontractivity) and generalizations of the standard Gaussian noise assumption on the subset of the data described by the unknown condition. Note that in regression, the error can be arbitrarily large with arbitrary probability, so bounds on the moments of the data are necessary to empirically estimate the error.

Our algorithm is inspired by the list-decodable subspace recovery algorithm presented by \cite{BK21}. Their work uses the sum of squares method to construct an algorithm which addresses robust subspace recovery. We present an analogous algorithm for conditional linear regression.  As in \cite{CJLLR20}, this is done by using a collection of subsets, which we will call ``terms,'' in place of individual points. We thus make use of the fact that by drawing many examples per term, the noise in the data selected by a term can be better controlled, leading to more accurate estimates. We stress that in contrast to the guarantee that Bakshi and Kothari obtain for their problem, we can obtain arbitrary accuracy with an algorithm that runs in fixed polynomial time, with an exponent that does not depend on the desired accuracy; we only require a sufficient (polynomial) number of examples from the target distribution, and our running time has only a low-order polynomial dependence on the size of the data set. (The dependence on the dimension, by contrast, while fixed, is a higher degree polynomial due to our use of the sum-of-squares method; reducing this dependence is a key challenge for future work, see Section~\ref{conclusions}.)

\subsection{Related Work}

Our work is both technically and conceptually related to ``list-decodable'' linear regression. Classical work in robust statistics~\citep{huber81,rl87} considers situations where a minority subset of the data consists of ``outliers'' that should be ignored. In this classical setting, it did not make sense to consider the possibility that a majority of the data could be ``outliers,'' in part because there would no longer be a unique, dominant solution to consider. But, a recently proposed model of ``list-decodable'' robust statistics~\citep{CSV17} (similarly to classical algorithms such as RANSAC~\citep{fb81} for subspace discovery) overcomes this obstacle by permitting a list of possible estimates or models to be produced, provided that the list is not too long (generally, $O(1/\mu)$ estimators for a $\mu$-fraction of the data) and that an accurate estimate appears somewhere in the list. In particular, algorithms for list-decodable linear regression have been proposed~\citep{RY20,kkk19} (see also \cite{BK21}). Although we have formulated our problem in such a way that we only produce a single arbitrary model as output, we could have returned a list of models as well (or vice-versa, select a suitable model from such a list). The main distinction is that in this line of work, on the one hand, one does not need to produce a DNF that identifies the inliers, in contrast to our setting. Note that without this formula, we cannot tell when we should use one of the models versus another to make predictions for new data. Of course, on the other hand, in these works one is not promised that such a DNF exists, either, and so the approach used in our analysis cannot be used in these problems.

Another similar line of work to conditional linear regression is \emph{selective regression}~\citep{eyw12}: here as well, the objective is to identify a fraction of the data that can be fit well. But in contrast to our setting, here a linear predictor is first identified, and then a data set is given, and finally a \emph{ranking} of that data is produced. The interpretation is that the top $\mu$-fraction of the ranking comprise the data for which the predictor is expected to be most accurate. In contrast, we jointly produce a linear model and a DNF that identifies which further examples, drawn from the same distribution, will be accurately predicted by the model. In \emph{learning with rejection} or \emph{abstention}~\citep{cdsm16}, on the other hand, a formula that selects a subset of the data is identified, but the problem formulation assigns a penalty to each example that is ``rejected'' -- thus, we have a default loss value that our classifier can take in place of the loss that would be incurred by this prediction, and this overall loss is minimized over the entire data set. The cost of rejection here takes the place of the probability of the subset $\mu$ in our problem.

All of these works have some similarities to classical topics such as fitting mixture models~\citep{mccs01,jiang07}. The primary difference is that in such work, first, \emph{every} data point should have been drawn from some linear model in the mixture; if some large subset of the data cannot be fit well by linear models, there is no guarantee that the model will identify a small subset that can be fit well. A second difference is that such models do not provide a (DNF) rule to decide whether or not subsequent data is drawn from one of the components versus another. There are a number of topics such as regression trees~\citep{quinlan92}, cluster-wise regression~\citep{PJKW17}, etc.\ that do provide such rules, but again, if the data overall cannot be fit well, they do not guarantee that small subsets of the data that can be fit well will be found.

\section{\uppercase{Preliminaries}}\label{preliminaries}

We will assume that we have a data set containing $N$ samples, from a distribution $\mathcal D$ over $\{0, 1\}^n \times \mathbb{R}^d \times \mathbb{R}$. Each sample consists of an $n$ dimensional vector of Boolean attributes $\boldsymbol{x}$, a $d$ dimensional real valued vector of predictor variables $\boldsymbol{y}$, and a real valued response, $z$, which we would like to predict. We will denote the $i$th sample as $(\boldsymbol{x}, \boldsymbol{y}, z)^{(i)}$ and abbreviate it as $\boldsymbol{x}^{(i)}$ when there is no ambiguity. 

For linear regression, we want to find a a vector of coefficients $\boldsymbol v$ such that $z$ can be predicted by $\langle \boldsymbol v, \boldsymbol y \rangle$. Typically, $\boldsymbol v$ is found using ordinary least squares, which minimizes the sum of $(\langle \boldsymbol v, \boldsymbol y \rangle - z )^2$ over all data points. However, since we are interested in cases where the majority of data cannot be fit, we want to find a subset of points described by condition $\boldsymbol c$ where there exists a good linear model.  Similar to previous work, we will consider conditions represented by Disjunctive Normal Form (DNF) formulas; other natural families of formulas are either weaker or yield intractable problems~\citep{J17}. A \emph{$k$-DNF} is defined to be a disjunction (OR) of \emph{terms} where each term is a conjunction (AND) of no more than $k$ attributes.

Throughout this paper, we will use $\| \cdot \|_F$ to denote the Frobenius norm and $\| \cdot \|_2$ to denote the $L2$ norm of a vector. We will also define $\mathcal X_I$ as the characteristic function where $\mathcal X_I(x) = 1$ if $x \in I$ and 0 otherwise. For brevity, we will use $[N] = \{n \in \mathbb N | 1 \leq n \leq N\}$. For a matrix $M$, $M\succeq 0$ denotes $M$ is positive semidefinite. Finally, we will use $\Pi$ to denote projection matrices.

\begin{definition}[Conditional Linear Regression]
Given a sample of $N$ points, $(\boldsymbol x, \boldsymbol y, z)^{(i)}$, from a distribution $\mathcal D$ over $\{0, 1\}^n \times \mathbb R^d \times \mathbb R$, the task of conditional linear regression is to find a $k$-DNF, $\boldsymbol c$, and linear predictor, $\boldsymbol v = (v_1 \dots v_d)^T$, such that, with high probability, $|\langle \boldsymbol v, \boldsymbol y\rangle - z|$ is bounded by $\epsilon$ when conditioned on $\boldsymbol c(\boldsymbol x) = 1$ and $\boldsymbol c(\boldsymbol x)=1$ is satisfied by at least a $\mu$ fraction of the data.
\end{definition}

We will present an algorithm that finds $\boldsymbol c$ and $\boldsymbol v$ that is close to the optimal values of $\boldsymbol c^\ast$ and $\boldsymbol v^\ast$ given that the distribution of samples conditioned on $\boldsymbol c^\ast$ follows certain regularity conditions. Our algorithm is obtained by solving a \emph{sum-of-squares relaxation} \citep{parrilo00,lasserre01,nesterov00,shor87} of a polynomial optimization problem:

\begin{definition}[Sum of squares relaxation]
Given a system of polynomial inequalities for polynomials in $\mathbb R[x_1,\ldots,x_n]$, $g_1(\boldsymbol{x})\geq 0,\ldots,g_r(\boldsymbol{x})\geq
0$, $h_1(\boldsymbol{x})=0,\ldots,h_s(\boldsymbol{x})=0$, the \emph{degree-$\ell$ sum-of-squares relaxation} is the following semidefinite optimization problem. 

The set of program variables $\boldsymbol{u}$ is indexed by monomials over $x_1,\ldots,x_n$ of total degree at most $\ell$, with $u_{\boldsymbol\alpha}$ denoting the variable for the monomial with degree vector $\boldsymbol\alpha\in\mathbb N^n$. 

We define the {\em degree-$\ell$ moment 
matrix} $M_\ell(\boldsymbol{u})$ indexed by $\boldsymbol\alpha,\boldsymbol\beta\in\mathbb N^n$ with total degree at most $\ell/2$ to be $M_\ell(\boldsymbol u)_{(\boldsymbol\alpha,
\boldsymbol\beta)}=u_{\boldsymbol\alpha+\boldsymbol\beta}$. 

For a ``shift'' polynomial $p\in\mathbb R[x_1,\ldots,
x_n]$ of degree $t$, letting $p_{\boldsymbol\gamma}$ denote the coefficient of the monomial $\boldsymbol{x}^{\boldsymbol\gamma}$ in $p$, the {\em degree-$\ell$ localizing matrix} $M_\ell (p\boldsymbol{u})$ is defined by $M_\ell (p\boldsymbol{u})_{(\boldsymbol\alpha,\boldsymbol\beta)}=\sum_{\boldsymbol\gamma} p_{\boldsymbol\gamma} u_{\boldsymbol\alpha+\boldsymbol\beta+\boldsymbol\gamma}$ for $\boldsymbol\alpha$ and $\boldsymbol\beta$ of total degree at most $\ell/2-t$. 

Now, the program has the constraints that $u_{\boldsymbol{0}}=1$; $M_\ell(\boldsymbol u)\succeq 0$; $M_\ell(h_j\boldsymbol{u})=0$ for $j\in\{1,\ldots,s\}$; and $M_\ell(g_j \boldsymbol{u})\succeq 0$ for 
$j\in\{1,\ldots,r\}$.
\end{definition}

Given a bound on the magnitudes of the values involved, the solutions to semidefinite programs can be approximated to arbitrary precision in polynomial time by various algorithms; the current state-of-the-art is due to \cite{jkl+20}.

A helpful interpretation of the sum-of-squares relaxation is that it defines a set of ``pseudo-distributions'' that relax the moments of probability distributions supported on solutions to the system of inequalities, with an associated ``pseudo-expectation'' operator defined on polynomials of degree up to $\ell$ (we borrow the presentation from \cite{RY20}):

\begin{definition}[Pseudo-distribution \citep{bbh+12}]
A level $\ell$ pseudo-distribution is a finitely-supported function $D: \mathbb R^n \rightarrow \mathbb R$ such that $\sum_x D(x) = 1$ and $\sum_x D(x) f(x)^2 \geq 0$ for every polynomial $f$ of degree at most $\ell/2$.
\end{definition}

\begin{definition}[Pseudo-expectation \citep{bbh+12}]
The pseudo-expectation of a function $f$ on $\mathbb R^d$ with respect to a pseudo-distribution $D$, denoted by $\tilde {\mathbb E}_{D(x)} = \sum_x D(x) f(x)$.
\end{definition}

A low-degree pseudo-distribution is generally \emph{not} a probability distribution, and we generally cannot sample from it.
The quality of a sum-of-squares relaxation is characterized by ``sum-of-squares \emph{proofs}''---under mild conditions, the bounds obtained by the relaxation of a given degree match the optimal bound that can be proved via a sum-of-squares proof \citep{parrilo00,lasserre01,nesterov00,shor87}:

\begin{definition}[Sum of Squares proofs \citep{gv02}]
Fix a set of polynomial inequalities $\mathcal A = \{g_i(x) \geq 0\}_{i \in [m]}\cup \{h_i(x) = 0\}_{i\in [m']}$ in variables $x_1, \dots, x_n$. A sum-of-squares proof of $q(x) \geq 0$ is an identity of the form

\begin{equation*}
\begin{split}
        & \left(1 + \sum_{k \in [m'']} d_k^2(x) \right) \cdot q(x) 
 = \sum_{j \in [m''']} s_j^2(x) + \sum_{i \in [m]} a_i^2(x) \cdot g_i(x) + \sum_{i\in [m']} b_i(x) h_i(x),
\end{split}
\end{equation*}
where $\{s_j(x)\}_{j \in [m''']}$, $\{a_i(x)\}_{i \in [m]}$, $\{b_i(x)\}_{i\in [m']}$, and $\{d_k(x)\}_{k \in [m'']}$ are real polynomials.
\end{definition}

Moreover, a recent technique for extracting solutions from the sum-of-squares relaxation has been to use \emph{identifiability} \citep{rss18}: roughly, if there is a sum-of-squares proof that a portion of the solution is determined up to small $\ell_2$ error, then we can read that portion of the solution off directly from the degree-1 variables. Conceptually, our analysis will follow this approach. Towards such an analysis, we require that the data/noise possess some niceness properties in a form that is amenable to sum-of-squares:

\begin{definition}[Certifiable Hypercontractivity, Section 2.1 of \cite{BK21}]
A distribution $\mathcal D$ over $\mathbb R^d$ has $C$-hypercontractive degree-2 polynomials if for every $d \times d$ matrix $Q$,
$$\mathbb E_{x \sim \mathcal D} \left(\boldsymbol x^\top Q \boldsymbol x - \text{tr}(Q)\right)^{2h} \leq (Ch)^h\left(\mathbb E_{\boldsymbol x \sim \mathcal D}\left(\boldsymbol x^\top Q \boldsymbol x\right)^2\right)^h.$$
We say that the hypercontractivity is \emph{$\ell$-certifiable} if there is a degree-$\ell$ sum-of-squares proof of this inequality with $Q$ as an indeterminate.
\end{definition}

\begin{lemma}[Certifiable Hypercontractivity Under Sampling. Lemma 6.11 in \cite{BK21}] \label{hypercontractivesampling}
Let $\mathcal D$ be a 1-subgaussian, $2h$-certifiably $C$-hypercontractive distribution over $\mathbb R^d$. Let $\mathcal S$ be a set of $n = \Omega((hd)^{8h})$ i.i.d. samples from $\mathcal D$. Then, with probability at least $1-1/\text{poly}(n)$, the uniform distribution on $\mathcal S$ is $h$-certifiably $(2C)$-hypercontractive.
\end{lemma}

In linear regression, it is standard and convenient to assume Gaussian residuals with mean 0 and constant variance $\sigma^2$. 
We stress that in order to empirically estimate the squared error of candidate linear models, we need some control on the moments of the loss function, which for various candidates corresponds to various quadratic polynomials; hypercontractivity of such quadratic polynomials gives such bounds.
Since many distributions, including uniform distributions, Gaussians, products of subgaussian distributions etc., are certifiably hypercontractive 
\citep{BK21}, it is reasonable to assume that we have samples from a certifiably hypercontractive distribution.
Lemma \ref{hypercontractivesampling} shows that empirical distributions from a certifiably hypercontractive distribution are also certifiably hypercontractive \citep{BK21}. We also assume that degree-2 polynomials similarly have \emph{certifiably} bounded variances, which also holds for such distributions \citep{BK21}.

\section{\uppercase{Results}}\label{results}

\subsection{Preprocessing}

For our algorithm, we will consider the data as $m$ disjoint subsets which we will call terms, $\{I_j\}_{j=1}^m$. To find $k$-DNF conditions, we will create a term for each setting of each set of $k$ distinct attributes, but the algorithm can be used with any family of sets we choose. We will weight each term by the number of points, $|I_j|$, in term $I_j$. We will define $I_{good}$ as the collection terms of the optimal $k$-DNF $\boldsymbol c^\ast$. From the perspective of the subspace recovery task \citep{BK21}, $I_{good}$ represents the collection of inliers, and the remaining points represent the outliers. Since the condition $\boldsymbol c^\ast$ is satisfied by a $\mu$ fraction of the data, $|I_{good}| = \mu N$. Additionally, our algorithm will assume that each term $I_j$ is pairwise disjoint. This can be achieved by duplicating points that satisfy more than one term \citep{CJLLR20}. We have $N$ data points and $m$ terms, so there will be at most $mN$ data points following the duplication procedure. Previously, $|I_{good}|=|\bigcup_{I_j \in I_{good}} I_j|$ which increases to $|I_{good}| = \sum_{I_j \in I_{good}} | I_j|$ with duplicate points. Therefore size of $I_{good}$ may blow up at most by a factor of the number of terms in $I_{good}$. If we use $N'$ to denote the number of points after duplication, notice that $N'_{good}/N' \geq N_{good}/mN$. Thus the proportion of good points after duplication is at least a $\mu/m$ \citep{CJLLR20}. Hereafter, $N$, $I_{good}$, $\{I_j\}_{j=1}^m$, and $\mu$ will be used to describe the data after duplication.

Our algorithm only obtains good estimates when each term in $I_{good}$ is large. Since the contribution of each term is weighted by its size, we can remove sufficiently small terms without compromising the quality of the estimates.
Finally, we will extend $\boldsymbol y^{(i)}$ by a constant, 1, to allow for an intercept in the linear model. 

\subsection{Main algorithm: identifying the data subspace}

Recall that our goal is to find the linear predictor $\boldsymbol v$ such that $\left\langle \boldsymbol v, \boldsymbol y^{(i)}\right\rangle = z^{(i)}$ for all points satisfying condition $\boldsymbol c^\ast$. If we extend $\boldsymbol v$ with -1 and $\boldsymbol y^{(i)}$ with $z^{(i)}$, the previous equation is equivalent to $\left\langle \boldsymbol v, \boldsymbol y^{(i)}\right\rangle = 0$. This equation describes a hyperplane, which allows us to look at our conditional linear regression problem from the perspective of a subspace recovery problem. Note that both $\boldsymbol v$ and $\boldsymbol y^{(i)}$ are now $(d+2)$-dimensional vectors. Let us use $\Pi$ to denote a projection matrix that projects onto this subspace and $\Pi_\ast$ to denote the optimal projection.

\begin{definition}[Reformulation of the Problem]
Given a distribution, $\mathcal D$, over points $\{\boldsymbol x^{(i)}, \boldsymbol y^{(i)}, z^{(i)} \}_{i=1}^N$ and predefined disjoint subsets, $\{I_j\}_{j=1}^m$. Let $I_{good}$ be an unknown collection of terms containing points that satisfy $\boldsymbol c^\ast$ such that $P(\boldsymbol x \in I_{good} \geq \mu)$. If there exists a linear predictor, $\boldsymbol v^\ast$, such that $|\langle \boldsymbol v^\ast, \boldsymbol y \rangle| \leq \epsilon$ and a projection matrix, $\Pi_\ast$, that projects onto the hyperplane described by the linear predictor, then we want to find $\hat \Pi$ that approximates $\Pi_\ast$ such that $\|\hat\Pi - \Pi_\ast \|_F \leq error$.
\end{definition}

Note that the Frobenius norm bounds the spectral norm from above, where the squared error of $v$ is precisely $v^\top\Pi_*v$, so a solution to this reformulation gives an empirical estimate that is off by at most $\|v\|_2^2\cdot error$. 
The advantage of this reformulation is that the subspace is uniquely determined by the set of terms, whereas if the subspace has lower dimension, there may be a large space of linear models, which are thus not identifiable. We remark that \cite{kkk19,RY20} made a stronger assumption to obtain identifiability for regression, that the data was \emph{anti-concentrated}. Anticoncentration does not hold if, e.g., the data lies on a lower-dimensional manifold, so it is a restrictive assumption. By contrast, we only need hypercontractivity and bounded moments to estimate the subspace.
We now state our main theorem, which is adapted from Theorem 1.4 of \cite{BK21}.

\begin{theorem} \label{mainthm}
Let $\Pi_\ast$ be a projection matrix for a subspace of dimension $r$. Let $\mathcal D|\boldsymbol{c^*}$ be a distribution with mean 0, covariance $\Pi_\ast$, and with $2$-certifiably $C$-hypercontracitve degree-2 polynomials with certifiably $C$-bounded variance. Then, there exists an algorithm that takes $n \geq \Omega\left((d \log d/\mu )^{16} \right)$ samples from the distribution conditioned on condition $\boldsymbol c^\ast$ and outputs a list $\mathcal L$ of $\mathcal O(1/\mu)$ projection matrices such that with probability at least 0.99 over the draw of the sample and randomness of the algorithm, there is a $\hat \Pi \in \mathcal L$ satisfying $\|\hat \Pi - \Pi_\ast \|_F \leq \mathcal O(1/\mu)$ in polynomial time.
\end{theorem}

Let $\mathcal A_{w, v, \epsilon, \Pi}$ be the SoS program in Figure \ref{program} where $\boldsymbol w$, $\boldsymbol v$, $\boldsymbol \epsilon$, and $\Pi_1, \dots, \Pi_m$ are indeterminates.

\begin{figure*}[t]
    \centering
    \begin{equation*}
    \begin{split}
        &\mathcal A_{w, v, \epsilon, \Pi}: \\
        &\left\{
        \begin{aligned}
        & & \boldsymbol y'^{(i)} =\yi - \frac{1}{\mu N}\sum_{j=1}^m \sum_{i=1}^N w_j &\Xij \yi + \begin{bmatrix}1 \\ \boldsymbol 0_{d+1}\end{bmatrix} \\
        \forall j &\in [m] & w_j \mathcal X_{I_j}(\boldsymbol x^{(i)}) \left(\left<\boldsymbol v, \boldsymbol y'^{(i)}\right> + \epsilon_i\right)&= 0 \\
        \forall j &\in [m] & \sum_{j=1}^m|I_j|w_j &= \mu N \\
        \forall j &\in [m] & w_j &= w_j^2 \\
        \forall j &\in [m] & w_j \mathcal X_{I_j}(\boldsymbol x^{(i)})\Pi_j \left( \boldsymbol y'^{(i)} - \begin{bmatrix}\boldsymbol 0_{d+1} \\ \epsilon_i \end{bmatrix}\right) &= w_j \mathcal X_{I_j}(\boldsymbol x^{(i)})\left(\boldsymbol y'^{(i)} - \begin{bmatrix}\boldsymbol 0_{d+1} \\ \epsilon_i \end{bmatrix}\right) \\
        \forall j &\in [m] & \sum_{i \in I_j} w_j \epsilon_i &\leq w_j \frac{ \sigma}{|I_j|} \\
        \forall Q, j &\in [m] & \frac{1}{\mu N}\sum_{i}^N w_j \mathcal X_{I_j}(\boldsymbol x^{(i)})\left({\boldsymbol y'^{(i)}}^\top Q\boldsymbol y'^{(i)} - \text{tr}(Q\Pi_j)\right)^{2} 
        &\leq \frac{Cw_j}{\mu N} \sum_{i}^N  \mathcal X_{I_j}(\boldsymbol x^{(i)}) ({\boldsymbol y'^{(i)}}^\top Q \boldsymbol y'^{(i)})^2\\
        \forall Q, j &\in [m] & \frac{1}{\mu N} \sum_{i}^N w_j \mathcal X_{I_j}(\boldsymbol x^{(i)}) (\boldsymbol y'^{(i)\top} Q \boldsymbol y'^{(i)})^2 &\leq C\|\Pi_jQ\Pi_j\|_F^{2} \\
        \forall j &\in [m]&  w_j \sum_{i \in I_j}\begin{bmatrix}{(\boldsymbol y_1'^{(i)}})^2 & \dots & (\boldsymbol y'^{(i)}_{d+2})^2 \end{bmatrix}^\top &\leq w_j |I_j| \alpha \boldsymbol 1 \\
        \forall j &\in [m]&  w_j \sum_{i \in I_j}\begin{bmatrix}{(\boldsymbol y_1'^{(i)}})^4 & \dots & (\boldsymbol y'^{(i)}_{d+2})^4 \end{bmatrix}^\top &\leq w_j |I_j| \beta  \boldsymbol 1
        \end{aligned}
        \right\}
    \end{split}
    \end{equation*}
    \caption{Polynomial constraints used in the SoS Program}
    \label{program}
\end{figure*}

We interpret the program constraints as follows:

\begin{compactenum}
    \item Defines $\boldsymbol y'^{(i)}$ such that the inliers have mean 0 while preserving the constant 1 for the intercept of the model.
    \item The linear predictor $\boldsymbol v$ fits well when conditioned on $\boldsymbol c$.
    \item The number of samples when conditioned on $\boldsymbol c$ comprises a $\mu$ fraction of the data.
    \item The Boolean constraint: $w_j \in \{0, 1\}$ for all $j$. 
    \item Defines $\Pi_j$ as a projection matrix corresponding to the distribution of points in term $I_j$.
    \item The residuals of the linear model follow a Gaussian distribution with mean 0 and standard deviation $\sigma$, which bounds the average noise on the inliers.
    \item\label{hc-constraint} The samples are certifiably hypercontractive.
    \item\label{var-constraint} Similarly, the variance is certifiably bounded.
    \item The second moment of each predictor variable of samples satisfying $\boldsymbol c$ is bounded by $\alpha$.
    \item Finally, the fourth moment of each predictor variable of samples satisfying $\boldsymbol c$ is bounded by $\beta$.
\end{compactenum}

We note that constraints \ref{hc-constraint} and \ref{var-constraint} are infinite families of constraints. But since there are sum-of-squares proofs of these constraints, similar to \cite{BK21}, we can use the quantifier elimination technique \cite[Section 4.3.4]{fkp19} to rewrite these as standard constraints. We note that we substitute the cubic polynomial $y'^\top Qy'$ for $Q$ in the proof of the hypercontractive inequality, thus obtaining a degree-6 sum-of-squares proof from the original degree-2 proof.

\begin{algorithm}
\caption{}\label{algorithm}
\begin{algorithmic}
\STATE {\bf Input:} Sample $\mathcal Y = \{\boldsymbol x^{(1)}, \boldsymbol x^{(2)}, \dots, \boldsymbol x^{(N)}\}$ sampled from the distribution conditioned on $\boldsymbol c^\ast$.
\STATE {\bf Output:} A list $\mathcal L$ of $O(1/\mu)$ projection matrices such that there exists $\hat \Pi \in \mathcal L$ satisfying $\|\hat\Pi - \Pi_\ast\|_F < \mathcal O(1/\mu)$.
\STATE {\bf Operation:} \\
\begin{enumerate}
    \item Find a degree 12 pseudo-distribution $\tilde \mu$ satisfying $\mathcal A_{w, v, \epsilon, \Pi}$ that minimizes $\sqrt{\sum_{j=1}^k w_j \sum_{\boldsymbol x^{(i)} \in I_j}\mathcal X_{I_j}(\boldsymbol x^{(i)})}$.
    \item For each $i \in [n]$ such that $\tilde{\mathbb E}_{\tilde \mu}\left[\sum_{j=1}^kw_j \mathcal X_{I_j}\left(x^{(i)}\right)\right] > 0$, let $\hat{\Pi}_i = \frac{\tilde{\mathbb E}_{\tilde \mu}\left[\sum_{j=1}^kw_j \mathcal X_{I_j}\left(\boldsymbol x^{(i)}\right) \Pi\right]}{\tilde{\mathbb E}_{\tilde \mu}\left[\sum_{j=1}^kw_j \mathcal X_{I_j}\left(\boldsymbol x^{(i)} \right)\right]}$. Otherwise, set $\hat {\Pi}_i = 0$.
    \item Take $J$ to be a random multi-set formed by union of $O(1/\mu)$ independent draws of $i \in [n]$ with probability $\frac{\tilde{\mathbb E}\left[\sum_{j=1}^kw_j \mathcal X_{I_j}\left(\boldsymbol x^{(i)} \right)\right]}{\mu N}$
    \item Output $L = \{\hat \Pi_i | i \in J\}$ where $J \subseteq [n]$
\end{enumerate}
\end{algorithmic}
\end{algorithm}

For the following analysis, $\yi$ will denote the samples after centering the inliers, which is equivalent to $\boldsymbol y'^{(i)}$ from the program.

\begin{lemma}[Frobenius Closeness of Empirical and True Covariances]\label{mainlemma}
Define $w(I_j) = \frac{1}{\mu N} w_j \Xij$. With probability $1-\delta$, 
\begin{equation*}
        \Bigg\|\frac{1}{\mu N} \sum_{i = 1}^N w_j \X_{I_j} (\boldsymbol x^{(i)})\Pi_j \yi \yit\Pi_j - w(I_j) \Pi_j \Bigg\|_F^2 \leq \bounda.
\end{equation*}
\end{lemma}

\begin{proof}
The squared Frobenius norm is equivalent to summing the square of each element. Thus we will bound the square of each element using Chebyshev's inequality and then compute the sum.
The matrix $\frac{1}{\mu N} \sum_{i = 1}^N w_j \Xij (\boldsymbol x^{(i)})\Pi_j \yi \yit\Pi_j$  can treated as a random quantity that represents an empirical estimate of $w(I_j) \Pi_j$. 

Recall $$w_j \Xij\Pi_j (\yi - \boldsymbol \epsilon_i) = w_j \Xij (\yi - \boldsymbol \epsilon_i)$$ from the program. Thus $$w_j \Xij \Pi_j \yi= w_j \Xij (\yi +  (\Pi_j - I)\boldsymbol \epsilon_i).$$ The empirical covariance matrix can be rewritten as $$\frac{1}{\mu N} \sum_{i = 1}^N w_j \Xij (\yi +  (\Pi_j - I)\boldsymbol \epsilon_i)(\yi +  (\Pi_j - I)\boldsymbol \epsilon_i)^\top.$$ Let us use the notation $A_{r,s}$ to denote the element in row $r$ and column $s$ of a matrix $A$. Using Chebyshev's Inequality, the square of each element in the Frobenius norm can be bounded by $$\frac{w_j(d+2)^2}{\mu ^2 N^2 \delta}\Big(\sum_{i=1}^N \Xij Var(( (\yi +  (\Pi_j - I)\boldsymbol \epsilon_i) (\yi +  (\Pi_j - I)\boldsymbol \epsilon_i)^\top)_{r,s})\Big)$$ with probability $1- \delta/(d+2)^2$.

Since $\yi$ has been extended with one, let us treat the first element of $\yi$ as 1. Therefore, when $r=1$ or $s=1$, $$Var(( (\yi +  (\Pi_j - I)\boldsymbol \epsilon_i)(\yi +  (\Pi_j - I)\boldsymbol \epsilon_i)^\top)_{r,s}) \leq \alpha + \alpha^2\sigma^2.$$ For all other pairs of $r$ and $s$, the variance is bounded by $\beta + 1\alpha^4$. Due to the same coordinate of $\yi$ being 1 for all $i \in [N]$, the second moment, $\alpha$, must be at least 1. Therefore, $$Var(( (\yi +  (\Pi_j - I)\boldsymbol \epsilon_i)(\yi +  (\Pi_j - I)\boldsymbol \epsilon_i)^\top)_{r,s}) \leq \beta + \alpha + 7\alpha^4\sigma^2.$$

Tying it all together, the square of each element can be bounded by $$\frac{w_j|I_j|w_j(d+2)^2}{\mu^2 N^2 \delta} \left(\beta + \alpha + 7\alpha^4\sigma^2\right)$$ with probability $1- \delta/(d+2)^2$. This is a $(d+2)$ dimensional square matrix so there are $(d+2)^2$ elements. Therefore, the squared Frobenius norm is bounded by $$\frac{w_j|I_j|w_j(d+2)^4}{\mu^2 N^2 \delta} \left(\beta + \alpha + 7\alpha^4\sigma^2\right)$$ and by a union bound, this holds with probability at least $1-\delta$.
\end{proof}

Lemma \ref{mainlemma}, which uses Chebyshev's Inequality, gives us the nice property of being inversely proportional to $N^2$. Therefore as we draw more data, the right hand side goes to 0.

\begin{lemma}[Frobenius Closeness of Subsample to Covariance, $w$-samples]\label{wsamples}
\begin{equation*}
    \begin{split}
        \mathcal A_{w, v, \epsilon, \Pi} \sststile{12}{w,v,\epsilon, \Pi} \!\Bigg\{  &\Bigg\|\frac{1}{\mu N}\!\sum_{i =1}^N w_j\Xij \yi \yit  
        \!\! - w(I_j) \Pi_j \Bigg\|_F^{4}\\
        &\leq C^2\!\left(\!\frac{1}{\mu N}\! \sum_{i=1}^N w_j \X_{I_j}(\boldsymbol x^{(i)} )\!\right)\! 
         \cdot\! \left(\!\bounda\! \right)\!\!\Bigg\}
    \end{split}
\end{equation*}
 with probability at least $1-\delta$.
\end{lemma}

\begin{lemma}[Frobenius Closeness of Subsample to Covariance, $I$-samples]\label{isamples}
\begin{equation*}
    \begin{split}
        \mathcal A_{w, v, \epsilon, \Pi} \sststile{12}{w, v,\epsilon, \Pi} \!\Bigg\{ \Bigg\|& \frac{1}{\mu N}\! \sum_{i =1}^N w_j\Xij \yi \yit 
        \!\! - w(I_j) \Pi_{j\ast} \Bigg\|_F^{4}\\
        \leq & C^2\!\left(\!\frac{1}{\mu N}\! \sum_{i=1}^N w_j \X_{I_j}(\boldsymbol x^{(i)} )\!\right) \cdot \left(\bounda \right)\!\!\Bigg\}
    \end{split}
\end{equation*}
 with probability at least $1-\delta$.
\end{lemma}

The proofs of Lemmas \ref{wsamples} and \ref{isamples} are similar to Lemmas 4.5 and 4.6 in \cite{BK21}, but by using Lemma \ref{mainlemma} we are able to form a tighter bound.

\begin{lemma}[Frobenius Closeness of $\Pi$ and $\Pi_\ast$, Lemma 4.3 in \cite{BK21}]\label{closenessofPis}
\begin{equation*}
    \begin{split}
        & \mathcal A_{w, v, \epsilon, \Pi} \sststile{12}{w,v,\epsilon,\Pi} \Bigg\{\! \left(\sum_{j=1}^m w(I_j)\!\right)\!\|\Pi - \Pi_\ast\|_F^{2}  \leq mC\sqrt{ 2^{5} \left(\boundb \right)} \Bigg\}.
    \end{split}
\end{equation*}
with probability at least $1-k\delta$ where $w(I_j) = \frac{1}{\mu N}\sum_{i}^N w_j \mathcal X_{I_j}\left(\boldsymbol x^{(i)}\right)$.
\end{lemma}

Lemma \ref{closenessofPis} admits a sum-of-squares proof and is proved by using Lemmas \ref{wsamples} and \ref{isamples}. The full proofs for all new lemmas are included in the supplementary material. 

\begin{lemma}[Large weight on inliers from high-entropy constraints. Fact 4.4 in \cite{BK21} and Lemma 3.1 in \cite{RY20}]\label{largeweightoninliers}
Let $\tilde{\mathbb E}_\xi$ be a pseudo-distribution of degree $\geq 2$ that satisfies $\mathcal A_{w, v, \epsilon, \Pi}$ and minimizes $\left\|\tilde{\mathbb E}_{\xi} \sum_{j=1}^m \sum_{i \in I_j} w_j \mathcal X_{I_j}\left(\boldsymbol x^{(i)}\right) \right\|_2$. Then $$\tilde {\mathbb E}_{\xi} \left[\sum_{j=1}^m \sum_{i \in I_j} w_j \mathcal X_{I_j} \left(\boldsymbol x^{(i)}\right)\right] \geq \mu^2 N.$$
\end{lemma}

The main theorem, Theorem \ref{mainthm}, can be obtained by using Lemmas \ref{closenessofPis} and \ref{largeweightoninliers} and following a similar argument to Theorem 1.4 of \cite{BK21}. 
\begin{proof}
Given a distribution $\mathcal D$ that is certifiably hypercontractive, Lemma \ref{hypercontractivesampling} implies that a large enough sample of inliers will also be certifiably hypercontractive with high probability. Algorithm 1 finds a pseudo-distribution that satisfies $\mathcal A_{w, v, \epsilon, \Pi}$ and minimizes $\sqrt{\sum_{j=1}^k w_j \sum_{\boldsymbol x^{(i)} \in I_j}\mathcal X_{I_j}(\boldsymbol x^{(i)})}$. Then $$\frac{1}{\mu M} \sum_{j}^k w_j \sum_{i =1}^N \tilde\E_{\tilde \mu} \left[\Xij \|\Pi - \Pi_\ast\|_F^2 \right]\leq \boundc$$ by using Lemma \ref{closenessofPis}. By applying Jensen's Inequality and taking the square root, we have $$\frac{1}{\mu N} \sum_{j =1}^m w_j \sum_{i =1}^N \tilde\E_{\tilde \mu}\left\|[\Xij \Pi] - [\Xij \Pi_\ast]\right\|_F^2 \leq \sqrt{\boundc}.$$ Due to the definition of $\hat \Pi_i$ from the algorithm, we can rewrite the inequality as $$\frac{1}{\mu N} \sum_{j =1}^m w_j \sum_{i =1}^N \tilde \E_{\tilde \mu} \|\hat\Pi_i - \Pi_\ast\|_F \leq \sqrt{\boundc}.$$ Let $Z = \frac{1}{\mu N}\sum_{j=1}^m w_j \sum_{i=1}^N \tilde \E [\Xij]$. Then, from Lemma \ref{largeweightoninliers}, $Z \geq \mu \Rightarrow \frac{1}{Z}\leq \frac{1}{\mu}$. Dividing by $Z$ on both sides thus yields: $$\frac{1}{Z}\left(\frac{1}{\mu N} \sum_{j =1}^m w_j \sum_{i =1}^N \tilde \E_{\tilde \mu} \|\hat\Pi_i - \Pi_\ast\|_F \right) \leq \frac{1}{\mu}\sqrt{\boundc}.$$

Since each index $i \in [N]$ is chosen with probability $$\frac{\widetilde{\mathbb E} [\sum_j^k w_j \mathcal X_{I_j}(x_i)]}{\sum_{i \in [n]}\widetilde{\mathbb E} [\sum_j^k w_j \mathcal X_{I_j}(x_i)]} = \frac{1}{\mu n} \widetilde{\mathbb E} [\sum_j^k w_j \mathcal X_{I_j}(x_i)],$$ it follows that $i \in I_{good}$ with probability at least $\frac{1}{\mu n} \sum_{j}^k w_j \sum_{i \in I_j} \widetilde{\mathbb E} [\mathcal X_{I_j}(x_i)] = Z \geq \mu$. By Markov's inequality applied to the last equation, with probability $\frac{1}{2}$ over the choice of $i$ conditioned on $i \in I_{good}$, $$\|\hat \Pi_i - \Pi_\ast\|_F \leq \frac{2}{\mu}\sqrt{\boundc}.$$ Thus, in total, with probability at least $\mu/2$, $$\|\hat \Pi_i - \Pi_\ast\|_F \leq \frac{2}{\mu}\sqrt{\boundc}.$$ 
Therefore, with probability of at least $0.99$ over the draw of the random set $J$, the list constructed by the algorithm contains $\hat \Pi$ such that $$\|\hat \Pi_i - \Pi_\ast\|_F \leq \frac{2}{\mu}\sqrt{\boundc}.$$

For the running time, the SDP can be solved in polynomial time and dominates the running time. Therefore, the algorithm runs in polynomial time overall.
\end{proof}

For a projection $\Pi$ and a linear predictor, $\boldsymbol v$, which satisfies $\langle \boldsymbol v, \boldsymbol y^{(i)}\rangle = 0$ for all $i \in I_{good}$ when disregarding noise, $\boldsymbol v^\top \Pi = \boldsymbol 0^\top$. Therefore, for each candidate $\Pi$, we can recover a candidate linear predictor $\boldsymbol u$ by treating $\boldsymbol u$ as a solution to the linear system $\boldsymbol u ^\top \Pi = \boldsymbol 0 ^\top$.

\subsection{Obtaining a $k$-DNF Condition}

Once we obtain an approximation for $\hat{\boldsymbol v}$, we will use the method described by \cite{CJLLR20} to obtain a $k$-DNF condition for $\boldsymbol c$. Since we have used the same reductions, specifically the reduction to disjoint terms by duplicating points, we will be able to invoke their analysis of their method.  Let us define a loss function, $f^{(i)} : \mathcal H \rightarrow \mathbb R$, for each point as $f^{(i)}(\boldsymbol v) = \left(z^{(i)} - \left<\boldsymbol v, \boldsymbol y^{(i)}\right>\right)^2$. Let $f_{I_j}(\boldsymbol v)$ be a the average loss over points in $I_j$. Finally, let us define $\bar f(\boldsymbol v) = \mathbb E[f_{I_{good}}(\boldsymbol v)]$ as the expectation of average loss for points in $I_{good}$. We will search through the Boolean data space $\{\boldsymbol x\}$ for conditions $\boldsymbol c$ for each $\boldsymbol u$. When we find a pair $(\boldsymbol u, \boldsymbol c)$ such that $f_c(\boldsymbol u)$ is small and there are are enough points satisfying $\boldsymbol c$, then we return that pair as a solution.  For some constant accuracy parameter $\gamma$ and Lipschitz constant $L$, if $\|\boldsymbol u - \boldsymbol v^\ast\| < \gamma$, then $|\bar f(\boldsymbol u) - \bar f(\boldsymbol v^\ast) \| \leq \gamma L= \mathcal O(\gamma)$. Since $\bar f$ is nonnegative, if $\bar f(\boldsymbol v^\ast) \leq \epsilon$, then $\bar f (\boldsymbol u) \leq \epsilon + \gamma$.

Recall that during preprocessing, we duplicated points that satisfied more than 1 term so that each term is disjoint. For $m$ terms, each point can be copied at most $m$ times. Let $t$ be the number of terms in $\boldsymbol c$.

\begin{lemma}[Lemma 3.4 in \cite{CJLLR20}]
Let $\boldsymbol u$ be such that $\|\boldsymbol u - \boldsymbol v^\ast\| < \gamma$. Then $|\bar f(\boldsymbol u)| \leq t(\gamma + \epsilon)$.
\end{lemma}

After obtaining $\boldsymbol u$ such that $f_i (\boldsymbol u)$ is close to $f_i (\boldsymbol v^\ast)$, \cite{CJLLR20} use a greedy set-cover algorithm to find the corresponding conditions $\boldsymbol c$. The algorithm greedily chooses terms $I_j$ satisfying $\sum_{i\in I_j} f^{(i)}(\boldsymbol u) \leq (1 + \gamma)\mu \epsilon N$ to maximize the number of additional points in $I_j$ that did not satisfy the previously chosen terms. It iterates until the number of points satisfying the chosen terms is at least $(1-\gamma/2)\mu N$ \citep{CJLLR20}. 

\begin{lemma}[Lemma 3.5 in \cite{CJLLR20}]
If there exists an optimal $k$-DNF $\boldsymbol c^\ast$ that is satisfied by a $\mu$-fraction of the points with total loss $\epsilon$, then, the weighted greedy set cover algorithm finds a $k$-DNF $\hat{\boldsymbol c}$ that is satisfied by a $(1 - \gamma) \mu$-fraction of the points with total loss $\mathcal O(t \log(\mu N)\epsilon)$.
\end{lemma}

Thus, we can obtain a pair $(\boldsymbol{u},\boldsymbol{c})$ that gives empirical error that is only greater than the optimal by a $\mathcal O(t \log(\mu N))$ factor. Given that our assumed bounds on the moments of the data distribution implies that the square of the loss (being a quadratic polynomial) is bounded, we can use the bounds of \cite{cgm13} to bound the generalization error of linear regression on each possible $k$-DNF, to thus obtain that the true generalization error is similarly bounded. For $\boldsymbol y \in \mathcal B \subseteq \mathbb R^{d}$, where $B$ is the $l_2$ radius of $\mathcal B$, this only incurs a polynomial increase in the sample complexity in $B$, $t$, and $1/\gamma$ overall. 

\section{\uppercase{Discussion and Future Directions}}\label{conclusions}
We have thus shown that the assumption of homogeneous covariances used by \cite{CJLLR20} is not needed to obtain a polynomial-time algorithm for conditional linear regression. On the other hand, although we obtain a polynomial running time and sample complexity, the exponents are quite large. In particular, the sample complexity we obtain in our analysis is far from optimal for this problem. Thus, our algorithm is impractical in its current form and our contribution is strictly theoretical. The main direction for future work is to develop a practical algorithm that does not require the homogeneous covariance assumption. 

We now elaborate on the obstacles.  Much of the overhead in the sample complexity arises from the use of the certifiable hypercontractivity assumption, which requires relatively high-degree polynomial expressions. Another source of sub-optimality seems to arise from the way we invoke bounds obtained via Chebyshev's inequality on the Frobenius error of the projector onto the data subspace. We conjecture that this can be improved by a more refined analysis, but it is still not clear whether or not certifiable hypercontractivity -- the dominant source of overhead for most purposes -- is really necessary.

The overhead in the computational complexity arises from the use of the sum-of-squares semidefinite program relaxation. Although the degree of the relaxation is moderate, the technique unfortunately yields algorithms that solve a large semidefinite program, and thus inherently tends to simply scale poorly. In many cases, however, the development of a sum-of-squares algorithm has led to the subsequent development of a spectral algorithm that can be practical. Examples of this sequence include a number of robust estimation tasks \citep{ss17,hss19,dkk20,depersin20} and tasks to identify hidden structures in a data set \citep{hsss16,hkp+17}, that are variously of similar flavor to our problem.

\section*{Acknowledgements}
 This research is partially supported by NSF awards IIS-1908287, IIS-1939677, and CCF-1718380.

\bibliographystyle{plainnat}
\bibliography{main}

\newpage

\appendix

\section{Sum of Squares}

\begin{lemma}[SoS H\"older's Inequality. Fact 3.11 in \cite{BK21}] \label{holders}
Let $w_1, \dots, w_n$ be indeterminates and let $f_1, \dots, f_n$ be polynomials of degree $m$ in vector valued variable $x$. Let $k$ be a power of 2. Then,
$$\left\{w_i^2 = w_i, \forall i \in [n]\right\} \sststile{2km}{x, w} \left\{\left(\frac{1}{n} \sum_{i=1}^n w_if_i\right)^k \leq \left(\frac{1}{n}\sum_{i=1}^n w_i\right)^{k-1}\left(\frac{1}{n} \sum_{i=1=}^n f_i^k\right)\right\}.$$
\end{lemma}

\begin{lemma}[SoS Almost Triangle Inequality. Fact 3.8 in \cite{BK21}] \label{triangle}
Let $a, b$ be indetermnates. Then, for any $t \in \mathbb N$,
$$\sststile{2t}{a,b} \left\{(a+b)^{2t} \leq 2^{2t} \left(a^{2t} + b^{2t} \right)\right\}.$$
\end{lemma}

\begin{lemma}[Cancellation within SoS. Fact 3.12 in \cite{BK21}] \label{cancellation}
Let $a$ be an indeterminate. Then,
$$\left\{a^t\leq 1\right\}\cup \{a \geq 0\} \sststile{t}{a}\{a \leq 1\}.$$
\end{lemma}

\begin{lemma}[Soundness. Fact 3.4 in \cite{BK21}]\label{soundness}
If $D \sdtstile{r}{}\mathcal A$ for a level-$\ell$. pseudo-distribution $D$ and there exists a sum-of-squares proof $\mathcal A \sststile{r'}{} \mathcal B$, then $D \sdtstile{r\cdot r' + r'}{} \mathcal B$.
\end{lemma}

\begin{lemma}[SoS Cauchy Schwarz. Fact 2.4 in \cite{RY20}]\label{cauchyschwarz}
Let $x_1, \dots, x_n, y_1, \dots, y_n$ be indeterminates, then 
$$\sststile{4}{}\left\{\left(\sum_{i=1}^n x_iy_i\right)^2 \leq \left(\sum_{1=1}^n x_i^2\right)\left(\sum_{i=1}^n y_i^2 \right)\right\}.$$
\end{lemma}

\begin{lemma}[SoS AM-GM Inequality. Fact 3.10 in \cite{BK21}] \label{AMGM}
Let $f1, f_2, \dots, f_m$ be intdeterminates. Then,
$$\sststile{m}{f_1, f_2, \dots, f_m} \left\{\left(\frac{1}{m}\sum_{i=1}^n f_i\right)^m \geq \Pi_{i \leq m}f_i\right\}.$$
\end{lemma}

\section{Analysis of Algorithm 1}\label{mainpf}

\begin{lemma}[Lemma 3.7 -- Large weight on inliers from high-entropy constraints. Fact 4.4 in  \cite{BK21} and Lemma 3.1 in \cite{RY20}]\label{Fact4.4}
Let $\tilde{\mathbb E}_\xi$ be a pseudo-distribution of degree $\geq 2$ that satisfies $\mathcal A_{w, v, \epsilon, \Pi}$ and minimizes $\left\|\tilde{\mathbb E}_{\xi} \sum_{j=1}^m \sum_{i \in I_j} w_j \mathcal X_{I_j}\left(\boldsymbol x^{(i)}\right) \right\|_2$. Then $\tilde {\mathbb E}_{\xi} \left[\sum_{j=1}^m \sum_{i \in I_j} w_j \mathcal X_{I_j} \left(\boldsymbol x^{(i)}\right)\right] \geq \mu^2 N$.
\end{lemma}

\begin{proof} (This proof is the same as Lemma 3.1 in \cite{RY20}.) For the sake of of simplicity, let $w_i = \sum_{i \in I_j} w_j \mathcal X_{I_j}\left(\boldsymbol x^{(i)}\right)$ and note that $w_i \in \{0, 1\}$. Let $\tilde{\mathbb E}_P$ denote a pseudo-distribution corresponding to the actual assignment $\{w_i'\}_{i \in [N]}$ and let $\tilde\E_D$ be the pseudo-expectation that minimizes $\|\E_D[w]\|$. For a constant $\kappa \in [0, 1]$, define the pseudo-expectation $\tilde\E_R$ as a mixture of $\tilde\E_P$ and $\tilde\E_D$.
$$\tilde{\mathbb E}_R \overset{def}{=} \kappa \tilde{\mathbb E}_P + (1 - \kappa) \tilde{\mathbb E}_D$$
Since $\E_D$ is the pseudo-expectation that minimizes $\|\tilde\E_D[w]\|$, then 
$$\left\langle\tilde\E_R[w], \tilde\E_R[w]\right\rangle \geq \left\langle\tilde\E_D[w], \tilde\E_D[w]\right\rangle.$$
We can use the definition of $\tilde \E_R$ to expand the left hand side.
$$\kappa^2 \left\langle\tilde\E_P[w], \tilde\E_P[w]\right\rangle + 2\kappa(1-\kappa)\left\langle\tilde\E_P[w], \tilde\E_D[w]\right\rangle + (1 - \kappa)^2 \left\langle\tilde\E_D[w], \tilde\E_D[w]\right\rangle \geq \left\langle\tilde\E_D[w], \tilde\E_D[w]\right\rangle$$
By rearranging the terms, we get
$$\left\langle\tilde\E_P[w], \tilde\E_D[w]\right\rangle \geq \frac{1}{2\kappa(1-\kappa)}\left((2\kappa - \kappa^2) \left\langle\tilde\E_D[w], \tilde\E_D[w]\right\rangle - \kappa^2 \left\langle\tilde\E_P[w], \tilde\E_P[w]\right\rangle \right).$$
By definition, $\left\langle\tilde\E_P[w], \tilde\E_P[w]\right\rangle = \sum_{i =1}^N w_i^2 = \mu N$. By using the Cauchy-Schwartz inequality, $\left\langle\tilde\E_D[w], \tilde\E_D[w]\right\rangle \geq \frac{1}{N} \left(\sum_{i}\tilde\E_D[w_i]\right)^2 = \frac{1}{N} (\mu N)^2 = \mu^2 N$. By substituting these bounds, we get that 
$$\left\langle\tilde\E_D[w], \tilde\E_P[w]\right\rangle \geq \frac{(2 \kappa - \kappa^2)\mu^2 - \kappa^2 \mu}{2\kappa(1-\kappa)} \cdot N.$$
As $\kappa \rightarrow 0$, the right hand side tends to $\mu^2N$.
\end{proof}

\begin{lemma}[Lemma 3.3 -- Frobenius Closeness of Empirical and True Covariances]\label{cheb}
Let $\boldsymbol \epsilon_i = \begin{bmatrix}\boldsymbol 0_{d+1} \\ \epsilon_i \end{bmatrix}$ and $w(I_j) = \frac{1}{\mu N} w_j \Xij$. Then with probability $1-\delta$, 
\begin{equation*}
    \begin{split}
        \left\|\frac{1}{\mu N} \sum_{i = 1}^N w_j \X_{I_j} (\boldsymbol x^{(i)})\Pi_j \yi \yit\Pi_j - w(I_j) \Pi_j \right\|_F^2\leq \bounda.
    \end{split}
\end{equation*}
\end{lemma}

\begin{proof}
The squared Frobenius norm is equivalent to summing the square of each element. Thus we will bound the square of each element using Chebyshev's inequality and then compute the sum.
The matrix $\frac{1}{\mu N} \sum_{i = 1}^N w_j \Xij (\boldsymbol x^{(i)})\Pi_j \yi \yit\Pi_j$  can treated as a random quantity that represents an empirical estimate of $w(I_j) \Pi_j$. 

Recall $$w_j \Xij\Pi_j (\yi - \boldsymbol \epsilon_i) = w_j \Xij (\yi - \boldsymbol \epsilon_i)$$ from the program. Thus $$w_j \Xij \Pi_j \yi = w_j \Xij (\yi +  (\Pi_j - I)\boldsymbol \epsilon_i).$$ The empircal covariance matrix can be rewritten as $$\frac{1}{\mu N} \sum_{i = 1}^N w_j \Xij (\yi +  (\Pi_j - I)\boldsymbol \epsilon_i)(\yi +  (\Pi_j - I)\boldsymbol \epsilon_i)^\top.$$ Let us use the notation $A_{r,s}$ to denote the element in row $r$ and column $s$ of a matrix $A$. Using Chebyshev's Inequality, the square of each element in the Frobenius norm can be bounded by $$\frac{w_j(d+2)^2}{\mu ^2 N^2 \delta}\left(\sum_{i=1}^N \Xij Var(( (\yi +  (\Pi_j - I)\boldsymbol \epsilon_i)(\yi +  (\Pi_j - I)\boldsymbol \epsilon_i)^\top)_{r,s})\right)$$ with probability $1- \delta/(d+2)^2$.

$Var(( (\yi +  (\Pi_j - I)\boldsymbol \epsilon_i)(\yi +  (\Pi_j - I)\boldsymbol \epsilon_i)^\top)_{r,s})$ is the variance of a sum of random variables. $$Var(\yi_r\yi_s) = E[\boldsymbol y_r^{(i)2}\boldsymbol y_s^{(i)2}] - E[\yi_r\yi_s]^2 \leq E[\boldsymbol y_r^{(i)2}\boldsymbol y_s^{(i)2}]$$ due to nonnegativity. By using the SoS Cauchy Schwarz and SoS AM-GM inequalities, Lemmas \ref{cauchyschwarz} and \ref{AMGM}, $$E[\boldsymbol y_r^{(i)2}\boldsymbol y_s^{(i)2}] \leq \sqrt{E[\boldsymbol y_r^{(i)2}]E[\boldsymbol y_s^{(i)2}]} \leq \left(E[\boldsymbol y_r^{(i)4}]+E[\boldsymbol y_s^{(i)4}]\right)/2.$$ Thus $$Var(\yi_r\yi_s) \leq \max\{E[\boldsymbol y_r^{(i)4}]E[\boldsymbol y_s^{(i)4}]\} \leq \beta.$$
Since $\yi$ has been extended with one, let us treat the first element of $\yi$ as 1. Therefore, when $r=1$ or $s=1$, $$Var(( (\yi +  (\Pi_j - I)\boldsymbol \epsilon_i)(\yi +  (\Pi_j - I)\boldsymbol \epsilon_i)^\top)_{r,s}) \leq \alpha + \alpha^2\sigma^2.$$ For all other pairs of $r$ and $s$, the variance is bounded by $\beta + 1\alpha^4$. Due to the same coordinate of $\yi$ being 1 for all $i \in [N]$, the second moment, $\alpha$, must be at least 1. Therefore, $$Var(( (\yi +  (\Pi_j - I)\boldsymbol \epsilon_i)(\yi +  (\Pi_j - I)\boldsymbol \epsilon_i)^\top)_{r,s}) \leq \beta + \alpha + 7\alpha^4\sigma^2.$$

Tying it all together, the square of each element can be bounded by $\frac{w_j|I_j|w_j(d+2)^2}{\mu^2 N^2 \delta} \left(\beta + \alpha + 7\alpha^4\sigma^2\right)$ with probability $1- \delta/(d+2)^2$. This is a $(d+2)$ dimensional square matrix so there are $(d+2)^2$ elements. Therefore, 
$$ \left\|\frac{1}{\mu N} \sum_{i = 1}^N w_j \X_{I_j} (\boldsymbol x^{(i)})\Pi_j \yi \yit\Pi_j - w(I_j) \Pi_j \right\|_F^2\leq \bounda.$$
By taking the union bound, this holds with probability at least $1-\delta$.
\end{proof}

\begin{lemma}[Lemma 3.4 -- Frobenius Closeness of Subsample to Covariance, $w$-samples. Same proof as Lemma 4.5 in \cite{BK21}] \label{wsamples_a}
\begin{equation*}
    \begin{split}
        \mathcal A_{w, v, \epsilon, \Pi} \sststile{12}{w, \Pi} \Bigg\{ & \left\|\frac{1}{\mu N} \sum_{i =1}^N w_j\Xij \yi \yit - w(I_j) \Pi_j \right\|_F^{4}\\
         & \leq C^2\left(\frac{1}{\mu N} \sum_{i=1}^N w_j \X_{I_j}(\boldsymbol x^{(i)} )\right) \left(\bounda \right)
    \end{split}
\end{equation*}
 with probability at least $1-\delta$.
\end{lemma}

\begin{proof}
For a $(d+2) \times (d+2)$ matrix-valued indeterminate $Q$ and using the SoS H\"older's Inequality, Lemma \ref{holders}, we have
\begin{equation} \label{BK4.4}
    \begin{split}
        \mathcal A_{w,v, \epsilon, \Pi} \sststile{12}{w, v, \epsilon, \Pi, Q} \Bigg\{ & \left<\frac{1}{\mu n}\sum_{i \in I_j}w_j\Xij \boldsymbol y^{(i)} \boldsymbol y^{(i)\top} - w(I_j)\Pi_j, Q\right>^{2} \\
        = & \left<\frac{1}{\mu N} \sum_{i=1}^N w_j \mathcal X_{I_j} (\boldsymbol x^{(i)})\left(\boldsymbol y^{(i)} \boldsymbol y^{(i)\top} - \Pi_j\right), Q\right>^{2}\\
        \leq & \left(\frac{1}{\mu N}\sum_{i=1}^N w_j \mathcal X_{I_j} (\boldsymbol x^{(i)})\right)\left(\frac{1}{\mu N}\sum_{i=1}^N w_j \mathcal X_{I_j}(\boldsymbol x^{(i)} ) \left<\boldsymbol y^{(i)} \boldsymbol y^{(i)\top} - \Pi_j, Q \right>^{2}\right)
        \Bigg\} 
    \end{split}
\end{equation}
Using certifiable hypercontractivity combined with the bounded variance constraints, we have
\begin{equation}\label{BK4.5}
    \begin{split}
        \mathcal A_{w,v, \epsilon,\Pi} \sststile{12}{w, v, \epsilon,\Pi, Q} \left\{ \frac{1}{\mu N} \sum_{i =1}^N w_j \X_{I_j} (\boldsymbol x^{(i)}) \left<\yi \yit - \Pi_j, Q \right>^{2} \leq \left(C^2t\right) \left\|\Pi_j Q \Pi_j \right\|_F^2 \right\}.
    \end{split}
\end{equation}
By combining Equations \ref{BK4.4} and \ref{BK4.5} and substituting $Q = \frac{1}{\mu N} \sum_{i = 1}^N w_j \X_{I_j}(\boldsymbol x^{(i)}) \yi \yit - w(I_j) \Pi$, we have
\begin{equation*}
    \begin{split}
        & \mathcal A_{w, v, \epsilon, \Pi} \sststile{12}{w, v, \epsilon, \Pi} \Bigg\{\left\|\frac{1}{\mu N} \sum_{i =1}^N w_j\Xij \yi \yit - w(I_j) \Pi_j \right\|_F^{4}\\
        & \leq C^2\left(\frac{1}{\mu N} \sum_{i=1}^N w_j \X_{I_j}(\boldsymbol x^{(i)} )\right)\left\|\Pi_j Q \Pi_j \right\|_F^{2}\\
        & = C^2\left(\frac{1}{\mu N} \sum_{i=1}^N w_j \X_{I_j}(\boldsymbol x^{(i)} )\right) \left\|\frac{1}{\mu N} \sum_{i = 1}^N w_j \X_{I_j} (\boldsymbol x^{(i)})\Pi_j \yi \yit\Pi_j - w(I_j) \Pi_j \right\|_F^{2}\Bigg\}.
    \end{split}
\end{equation*}
By using Lemma \ref{cheb}, we have
\begin{equation*}\label{BK4.6}
    \begin{split}
        \mathcal A_{w, v, \epsilon, \Pi} \sststile{12}{w, v, \epsilon, \Pi} \Bigg\{& \left\|\frac{1}{\mu N} \sum_{i =1}^N w_j\Xij \yi \yit - w(I_j) \Pi_j \right\|_F^{4}\\
         & \leq C^2 \left(\frac{1}{\mu N} \sum_{i=1}^N w_j \X_{I_j}(\boldsymbol x^{(i)} )\right)  \left(\bounda \right)
    \end{split}
\end{equation*}
with probability at least $1 - \delta$.
\end{proof}

\begin{lemma}[Lemma 3.5 -- Frobenius Closeness of Subsample to Covariance, $I$-samples. Same proof as Lemma 4.5 in \cite{BK21}] \label{isamples_a}
\begin{equation*}
    \begin{split}
        \mathcal A_{w, v, \epsilon, \Pi} \sststile{12}{w, v, \epsilon, \Pi} \Bigg\{ &\left\|\frac{1}{\mu N} \sum_{i =1}^N w_j\Xij \yi \yit - w(I_j) \Pi_{j\ast} \right\|_F^{4}\\
        &\leq C^2 \left(\frac{1}{\mu N} \sum_{i=1}^N w_j \X_{I_j}(\boldsymbol x^{(i)} )\right)  \left(\bounda \right)
    \end{split}
\end{equation*}
 with probability at least $1-\delta$.
\end{lemma}

\begin{proof}
This follows the same proof as Lemma \ref{wsamples_a}.
\end{proof}

\begin{lemma}[Lemma 3.6 -- Frobenius Closeness of $\Pi$ and $\Pi_\ast$. Same as Lemma 4.3 in \cite{BK21}.] \label{BK4.3}
\begin{equation*}
    \begin{split}
        \mathcal A_{w, v, \epsilon, \Pi} \sststile{12}{w,v, \epsilon,\Pi} \Bigg\{ & \left(\sum_{j=1}^m w(I_j)\right)\|\Pi - \Pi_\ast\|_F^{2} 
        \leq mC \sqrt{ 2^{5}  \left(\boundb \right)} \Bigg\}.
    \end{split}
\end{equation*}
with probability at least $1-k\delta$ where $w(I_j) = \frac{1}{\mu N}\sum_{i}^N w_j \mathcal X_{I_j}\left(\boldsymbol x^{(i)}\right)$.
\end{lemma}

\begin{proof}
Define $w^\ast(I_j) = \frac{|I_j|}{\mu N}$. Using the SoS Almost Triangle Inequality, Lemma \ref{triangle}, and Lemmas \ref{wsamples_a} and \ref{isamples_a}, we have
\begin{equation*}
    \begin{split}
        \mathcal A_{w, v, \epsilon, \Pi} \sststile{12}{w,v, \epsilon,\Pi}\Bigg\{  &  w(I_j)^{4} \left\|\Pi_j - \Pi_{j \ast} \right\|_F^{4} 
        \leq 2^{5}C^2 w(I_j) \left(\bounda \right) \Bigg\}.
    \end{split}
\end{equation*}
By dividing both sides of the inequality by $w(I_j)^{2}$, we have
\begin{equation*}
    \begin{split}
        \mathcal A_{w, v, \epsilon,\Pi} \sststile{12}{w,v, \epsilon,\Pi}  \Bigg\{& w(I_j)^{2} \left\|\Pi_j - \Pi_{j \ast} \right\|_F^{4} 
        \leq  2^{5}C^2 w(I_j)^{-1}  \left(\bounda\right) \Bigg\}.
    \end{split}
\end{equation*}
By using Cancellation within SoS, Lemma \ref{cancellation} and multiplying both sides of the inequality by $w(I_j)^{1/2}$, we have
\begin{equation*}
    \begin{split}
        \mathcal A_{w, v, \epsilon,\Pi} \sststile{12}{w,v, \epsilon,\Pi} \Bigg\{  &w(I_j)w(I_j)^{1/2} \left\|\Pi_j - \Pi_{j \ast} \right\|_F^{2} 
        \leq C \sqrt{ 2^{5} \left(\bounda \right)} \Bigg\}.
    \end{split}
\end{equation*}
Since $w(I_j) = \frac{1}{\mu N} \sum_{i=1}^N w_j \Xij$ and is bounded above by $w^\ast(I_j)$, then with probability $1-2\delta$:
\begin{equation*}
    \begin{split}
        \mathcal A_{w, v, \epsilon,\Pi} \sststile{12}{w,v, \epsilon,\Pi} \Bigg\{& w(I_j) \left\|\Pi_j - \Pi_{j \ast} \right\|_F^{2} \leq C \sqrt{ 2^{5} \left(\boundb \right)} \Bigg\}.
    \end{split}
\end{equation*}
Define $\Pi$ and $\Pi_\ast$ as a weighted average of $\Pi_j$ and $\Pi_{j\ast}$ respectively, where the weights are proportional to $|I_j|$. Thus $\Pi = \sum_{j=1}^m w(I_j) \Pi_j$ and $\Pi_\ast = \sum_{j=1}^m w(I_j) \Pi_{j\ast}$. By summing both sides of the inequality over all $j \in [m]$, we have
\begin{equation*}
    \begin{split}
        \mathcal A_{w, v, \epsilon, \Pi} \sststile{12}{w,v, \epsilon,\Pi} \Bigg\{&  \sum_{j=1}^m w(I_j)\left\|\Pi_j - \Pi_{j \ast} \right\|_F^{2} 
        \leq \sum_{j=1}^m C \sqrt{ 2^{5} \left(\boundb \right)} \Bigg\}.
    \end{split}
\end{equation*}
By using the Cauchy-Schwarz Inequality and the triangle inequality to rewrite the left hand side, we have
\begin{equation*}
    \begin{split}
        \mathcal A_{w, v, \epsilon, \Pi} \sststile{12}{w,v, \epsilon,\Pi} \Bigg\{ & \left(\sum_{j=1}^m w(I_j)\right)\|\Pi - \Pi_\ast\|_F^{2} 
        \leq mC \sqrt{ 2^{5} \left(\boundb \right)} \Bigg\}
    \end{split}
\end{equation*}
with probability $1 - 2m\delta$.
\end{proof}

\begin{lemma}[Lemma 2.7 -- Certifiable Hypercontractivity Under Sampling. Lemma 6.11 in \cite{BK21}] \label{BK6.11}
Let $\mathcal D$ be a 1-subgaussian, $2h$-certifiably $C$-hypercontractive distribution over $\mathbb R^d$. Let $\mathcal S$ be a set of $n = \Omega((hd)^{8h})$ i.i.d. samples from $\mathcal D$. Then, with probability at least $1-1/\text{poly}(n)$, the uniform distribution on $\mathcal S$ is $h$-certifiably $(2C)$-hypercontractive.
\end{lemma}

\subsection{Proof of Main Theorem} 

\begin{theorem}[Main Theorem -- Theorem 3.2]
Let $\Pi_\ast$ be a projection matrix for a subspace of dimension $r$. Let $\mathcal D$ be a distribution with mean 0, covariance $\Pi_\ast$, and 2-certifiably $C$-hypercontracitve degree-2 polynomials. Then, there exists an algorithm that takes $n \geq \Omega\left((d \log d/\mu )^{16} \right)$ samples from the distribution conditioned on condition $\boldsymbol c^\ast$ and outputs a list $\mathcal L$ of $O(1/\mu)$ projection matrices such that with probability at least 0.99 over the draw of the sample and randomness of the algorithm, there is a $\hat \Pi \in \mathcal L$ satisfying $\|\hat \Pi - \Pi_\ast \|_F \leq O(1/\mu)$ in polynomial time.
\end{theorem}

\begin{proof}
This follows the same proof as Theorem 1.4 in \cite{BK21}. Since $\mathcal D$ is certifiably $C$-hypercontractive, Fact \ref{BK6.11} implies that $\geq n=\Omega( d \log d/\mu)^{16}$ samples suffice for the uniform distribtution on the inliers, $I_{good}$, to have 2-certifiably $C$-hypercontractive degree 2 polynomials with probabiliy at least $1-1/d$. Let $\xi_1$ be the event that this succeds, and condition on it.

Let $\tilde \mu$ be a pseudo-distribution of degree-24 satisfying $\mathcal A_{w,v, \epsilon, \Pi}$,  minimizing $\sqrt{\sum_{j=1}^m w_j \sum_{i=1}^N \Xij }$ as described in Algorithm 1. Observe that such a pseudo-distribution is guaranteed to exist: take the pesudo-distribution supported on a single point, $(w, \Pi)$ such that $w_i=1$ iff $i \in I_{good}$ and $\Pi = \Pi_\ast$. It is straight forward to check that $\Pi_\ast$ is indeed a rank $r$ projection matrix and rank $d+2$ projection matrix and $\sum_{j=1}^m \sum_{i=1}^N \Xij=\mu N$. Conditioned on $\xi_1$, the hypercontractivity constraint is also satisfied by the inliers. 

Since Lemma \ref{BK4.3} admits a sum-of-squares proof, it follows from Fact \ref{soundness} that the polynomial inequality is preserved under pseudo-expectations. 
\begin{equation*}
    \begin{split}
        \frac{1}{\mu M} \sum_{j}^k w_j \sum_{i =1}^N \tilde\E_{\tilde \mu} \left[\Xij \|\Pi - \Pi_\ast\|_F^2 \right]  \leq \boundc.
    \end{split}
\end{equation*}

Alternatively, we can rewrite the above as follows:
\begin{equation*}
    \begin{split}
        &\frac{1}{\mu N} \sum_{j =1}^m w_j \sum_{i =1}^N \tilde\E_{\tilde \mu}\left\|[\Xij \Pi] - [\Xij \Pi_\ast]\right\|_F^2 
        \leq \boundc.
    \end{split}
\end{equation*}
Applying Jensen's Inequality yields
\begin{equation*}
    \begin{split}
        &\left(\frac{1}{\mu N} \sum_{j =1}^m w_j \sum_{i =1}^N \left\|\tilde\E_{\tilde \mu}[\Xij \Pi] - \tilde \E_{\tilde\mu}[\Xij \Pi_\ast]\right\|_F \right)^2 \\
        & \leq \boundc.
    \end{split}
\end{equation*}
Taking the square root,
\begin{equation*}
    \begin{split}
        &\frac{1}{\mu N} \sum_{j =1}^m w_j \sum_{i =1}^N \left\|\tilde\E_{\tilde \mu}[\Xij \Pi] - \tilde \E_{\tilde\mu}[\Xij \Pi_\ast]\right\|_F 
         \leq \sqrt{\boundc}.
    \end{split}
\end{equation*}
Recall, the rounding in Algorithm 1 uses $\hat{\Pi}_i = \frac{\tilde{\mathbb E}_{\tilde \mu}\left[\sum_{j=1}^kw_j \mathcal X_{I_j}\left(\boldsymbol x^{(i)}\right) \Pi\right]}{\tilde{\mathbb E}_{\tilde \mu}\left[\sum_{j=1}^kw_j \mathcal X_{I_j}\left(\boldsymbol x^{(i)} \right)\right]}$ to denote the projector corresponding to the $i$-th sample.
Then rewriting the above equation yields:
\begin{equation*}
    \begin{split}
        & \frac{1}{\mu N} \sum_{j =1}^m w_j \sum_{i =1}^N \tilde \E_{\tilde \mu} \|\hat\Pi_i - \Pi_\ast\|_F \leq \sqrt{\boundc}.
    \end{split}
\end{equation*}
Let $Z = \frac{1}{\mu N}\sum_{j=1}^m w_j \sum_{i=1}^N \tilde \E [\Xij]$. Then, from Lemma \ref{Fact4.4}, $Z \geq \mu \Rightarrow \frac{1}{Z}\leq \frac{1}{\mu}$. Dividing by $Z$ on both sides thus yields:
\begin{equation*}
    \begin{split}
        & \frac{1}{Z}\left(\frac{1}{\mu N} \sum_{j =1}^m w_j \sum_{i =1}^N \tilde \E_{\tilde \mu} \|\hat\Pi_i - \Pi_\ast\|_F \right) \leq \frac{1}{\mu}\sqrt{\boundc}.
    \end{split}
\end{equation*}

Since each index $i \in [N]$ is chosen with probability $$\frac{\widetilde{\mathbb E} [\sum_j^k w_j \mathcal X_{I_j}(x_i)]}{\sum_{i \in [n]}\widetilde{\mathbb E} [\sum_j^k w_j \mathcal X_{I_j}(x_i)]} = \frac{1}{\mu n} \widetilde{\mathbb E} [\sum_j^k w_j \mathcal X_{I_j}(x_i)],$$ it follows that $i \in I_{good}$ with probability at least $\frac{1}{\mu n} \sum_{j}^k w_j \sum_{i \in I_j} \widetilde{\mathbb E} [\mathcal X_{I_j}(x_i)] = Z \geq \mu$. By Markov's inequality applied to the last equation, with probability $\frac{1}{2}$ over the choice of $i$ conditioned on $i \in I_{good}$, $$\|\hat \Pi_i - \Pi_\ast\|_F \leq \frac{2}{\mu}\sqrt{\boundc}.$$ Thus, in total, with probability at least $\mu/2$, $$\|\hat \Pi_i - \Pi_\ast\|_F \leq \frac{2}{\mu}\sqrt{\boundc}.$$ 
Therefore, with probability of at least $0.99$ over the draw of the random set $J$, the list constructed by the algorithm contains $\hat \Pi$ such that $$\|\hat \Pi_i - \Pi_\ast\|_F \leq \frac{2}{\mu}\sqrt{\boundc}.$$

Now to account for the running time of the algorithm, the SDP for the program can be solved in polynomial time, so the algorithm runs in polynomial time overall.
\end{proof}

\end{document}